\documentclass{article}

\usepackage{arxiv}
\usepackage{fullpage}
\usepackage[utf8]{inputenc}
\usepackage[T1]{fontenc}
\usepackage{hyperref}
\usepackage{url}
\usepackage{amsfonts}
\usepackage{clrscode}
\usepackage{physics}
\usepackage{graphicx}
\usepackage{caption}
\usepackage{float}
\usepackage{subcaption}
\usepackage{siunitx}
\usepackage{mathtools}
\usepackage{comment}
\usepackage[title]{appendix}

\usepackage{amsmath,amsfonts,amssymb,amsthm}
\usepackage{nicefrac}
\usepackage{booktabs}
\theoremstyle{plain}
\newtheorem{thm}{Theorem}[section]
\newtheorem{lem}[thm]{Lemma}
\newtheorem{prop}[thm]{Proposition}

\theoremstyle{definition}
\newtheorem{defn}{Definition}[section]

\newtheorem{assum}{Assumption}[section]

\theoremstyle{remark}

\newtheorem{theorem}{Theorem}
\newtheorem*{theorem*}{Theorem}
\newtheorem*{lem*}{Lemma}
\newtheorem*{thm*}{Theorem}
\newtheorem*{defn*}{Definition}

\DeclareMathOperator*{\argmin}{arg\,min}

\hypersetup{
    colorlinks=true,
    linkcolor=blue,
    filecolor=magenta,      
    urlcolor=cyan,
}

\title{Estimation and Applications of Quantiles in Deep Binary Classification}

\author{
  Anuj Tambwekar \\
  Department of CS\&E \\
  PES University\\
  Bengaluru, Karnataka, India \\
  \texttt{anujstam@gmail.com} \\
   \And
  Anirudh Maiya \\
  Department of CS\&E \\
  PES University\\
  Bengaluru, Karnataka, India \\
  \texttt{maiyaanirudh@gmail.com} \\
   \And
    Soma Dhavala \\
  Founder\\
  MLSquare\\
  Bengaluru, India \\
  \texttt{soma@mlsquare.org} \\
    \And
  Snehanshu Saha \\
  Department of CSIS and APPCAIR \\
  Birla Institute of Technology and Science\\
  Goa, India \\
  \texttt{snehanshus@goa.bits-pilani.ac.in} \\

}

\begin{document}
\date{}
\maketitle

\begin{abstract}
Quantile regression, based on check loss, is a widely used inferential paradigm in Econometrics and Statistics. The conditional quantiles provide a robust alternative to classical conditional means, and also allow uncertainty quantification of the predictions, while making very few distributional assumptions. We consider the analogue of check loss in the binary classification setting. We assume that the conditional quantiles are smooth functions that can be learnt by Deep Neural Networks (DNNs). Subsequently, we compute the Lipschitz constant of the proposed loss, and also show that its curvature is bounded, under some regularity conditions. Consequently, recent results on the error rates and DNN architecture complexity become directly applicable.

We quantify the uncertainty of the class probabilities in terms of prediction intervals, and develop individualized confidence scores that can be used to decide whether a prediction is reliable or not at scoring time. By aggregating the confidence scores at the dataset level, we provide two additional metrics, model confidence, and retention rate, to complement the widely used classifier summaries. We also the robustness of the proposed non-parametric binary quantile classification framework are also studied, and we demonstrate how to obtain several univariate summary statistics of the conditional distributions, in particular conditional means, using smoothed conditional quantiles, allowing the use of explanation techniques like Shapley to explain the mean predictions. Finally, we demonstrate an efficient training regime for this loss based on Stochastic Gradient Descent with Lipschitz Adaptive Learning Rates (LALR). 
\end{abstract}

\section{Introduction}
Deep Learning has seen tremendous success over the last few years in the fields of Computer Vision, Speech and Natural Language processing \cite{hinton}. As it is making its way into numerous real world applications, focus is shifting from achieving state-of-the-art performance to questions about explainability, robustness, trustworthiness, fairness and training efficiency, among others. Uncertainty Quantification (UQ) is also witnessing a renewed interest within the Deep Learning community. All of the aforementioned aspects need to be tackled in a holistic manner to democratize AI and make AI equitable for all sections of the society at large \cite{Ahmed2020AFF}. 
In a seminal paper, Parzen provides a foundation for exploratory and confirmatory data analysis using quantiles \cite{parzen_79}, and later argues for unification of the theory and practice of statistical methods with them \cite{parzen_unified}. In this work, we take these ideas forward and show how some of the problems mentioned before in the Deep Learning context, can be solved using a quantile-centric approach.

Quantile Regression (QR) generalizes the traditional mean regression to model the relationship between the quantiles of the response to the dependent variables, median regression being the special case \cite{koenker_78}. QR inherits many desirable properties of quantiles: they are robust to noise in the response variable, have a clear probabilistic interpretation, and are equivariant under monotonic transformations. Besides, they also have appealing asymptotic properties under mild assumptions both in the parametric and the non-parametric settings \cite{Portnoy_89,probal_doksum}. QR has found many successful applications in Econometrics and Statsitics, such as modeling growth curves, extreme events, and in the robust regression contexts  \cite{koenker_2005, trivedi_2010, Maronna2006RobustST, probal_l1}. Its introduction to the Machine Learning community is relatively recent, where \cite{smola} showed the relationship between $\nu-$Support Vector Machines and QR, for example. \cite{natasa} applied QR for modeling \textit{aleatoric} uncertainty in deep learning via prediction intervals (PIs). Unlike previous works, we study QR in the binary classification setting since: 1) Much of the earlier work on QR can be extended to the deep learning context, with very minimal effort, which is not the case with classification tasks 2) Despite the dominance of classification tasks in the DL space, reliance on the popular but problematic binary cross entropy is still prevalent, and there is a need to find viable alternatives 3) We also want to study several problems together, as mentioned before, with binary classification as a test bed. We hope that, our findings can be extended to multi-class settings in future.

In the rest of this work, first we setup the problem, along with notations, and derive the Binary Quantile Regression (BQR) loss. We derive some properties of the loss function and provide the learning rates under the regularity assumptions. Later, for each of the sub-problems, namely, UQ, Explainability, Robustness, and Adaptive Learning Rates, we provide the necessary background, develop the idea, and provide the results. 
Finally, we discuss our findings, scope for improvements and new opportunities.

\section{Binary Quantile Regression}

\subsection{Setup and Notations}
\begin{defn}
For any real valued random variable $Z$, with distribution function $F(z)$, with $F(z) = P(Z \le z) $, the quantile function $Q(\tau)$ is given as
$Q(\tau) = F^{-1}(\tau) = \inf\{r: F(r) \ge \tau\} $ for any $0 < \tau < 1 $.
\end{defn}
\begin{assum}
We collect n i.i.d samples $\{x_i,y_i\}^{n}_{i=1}$, where $x \in [-1,1]^d$, and is continuously distributed, represents the d-dimensional input features and $y \in \{0,1\}$ the class label. For an absolute constant $M > 0$, assume $\|f^*\|_{\infty} \le M$ 
\end{assum}
\begin{assum}
Assume $f^*$ lies in the Sobolev ball $W^{\beta,\infty}([-1, 1]^d)$, with smoothness $\beta \in  N_{+}$
\begin{eqnarray*}
    f^{*}(x) &\in& W^{\beta,\infty}([-1, 1]^d): = \left\{ f:  \max_{\alpha, |\alpha|\le\beta}  \sup{x \in [-1,1]^{d}} |D^{\alpha}|\le 1 \right\},
\end{eqnarray*}
where $\alpha =(\alpha_1,\alpha_2,\hdots, \alpha_d), |\alpha|=\alpha_1+\alpha_2,\hdots+\alpha_d$ and $D^{\alpha}f$ is the weak derivative.
\end{assum}
\begin{assum}
Let $f^*$ lie in a class F. For the feedforward network class $F_{DNN}$, let the approximation error $\epsilon_{}$ be
$$
\epsilon_{f^*} :=\sup_{f^* \in F} \inf_{\substack{ f \in F_{DNN} \\ \|f\|_{\infty}\le 2M}} {\|f - f^*\|}_{\infty}
$$
\end{assum}
Given the input features $x$, we aim to learn a classifier that maps the inputs to the class labels. Let $Q_x{(\tau)}=f_{\tau}(x), \tau \in (0,1)$ be a continuous, smooth, conditional (on x) quantile learnt by the DNN.  We consider an architecture of the form
 $Q_x{(\tau)} = g_{\tau}(g_{c}(x))$ where $g_{\tau}$ is quantile-specific network and $g_{c}$ is a layer shared by all quantiles. \cite{zhu_cqr} showed that, sharing parameters across quantiles generally leads to better statistical efficiency. It is akin to multi-task learning, where each quantile estimation is a task, and our architecture is inspired by this observation.
 \subsection{Background}
\cite{mansky_75} considered the median regression for thresholded binary response models of the form
$Z = x\beta + U, Y = I(Z\ge 0)$, where $Z$ is the latent response, $\beta$ is a $d \times 1$ vector of unknowns, $U \sim F(.)$ are i.i.d errors from a continuous distribution and $I(.)$ is an indicator function. Later, in \cite{mansky_85}, he proved the consistency and asymptotic properties of the Maximum Score Estimator and also noted that it can be extended to model other quantiles as a solution to the to the optimization problem:
$\argmin_{\beta |\beta|=1} \sum_{i=1}^{n} \rho_{\tau}(y_i-I(x\beta\ge0))$. Here, $\rho_{\tau}$ is the check loss or pinball loss, defined as $\rho_{\tau}(e) = (\tau - I(e<0))e$ \cite{koenker_78}. It is well known that check loss is a generalized version of the Mean Absolute Error (MAE), often used in Robust Regression settings, and that quantiles minimize the check loss. \cite{horowitz_92, kordas_2006} provided efficient estimators by replacing the Indicator function with smooth kernels. \cite{Benoit_12} considered the Bayesian counterpart by noting that check loss is the kernel of Asymmetric Laplace Density (ALD). Below, we extend this to the non-parametric settings and derive the loss function suitable for DNNs.
\subsection{BQR Loss}
Let us reconsider thresholded binary response model $$y=I(z\ge0), \, z=Q_x(\tau)=f_{\tau}(x) + \epsilon$$ where $\epsilon \sim ALD(0,1,\tau)$ and $$ALD(y;\mu,\sigma,\tau) \equiv \tau(1-\tau){\sigma}^{-1} \exp(-\rho_{\tau}((y-\mu){\sigma}^{-1}))$$
It can be shown that, 
\begin{eqnarray*}
    P(y=1 | f_{\tau}(x)) &\equiv&
    \begin{cases} 
      1 - \tau \exp((\tau-1)f_{\tau}(x)) & 0 < f_{\tau}(x) \\
      (1-\tau)\exp(\tau f_{\tau}(x)) & 0 \geq f_{\tau}(x) \\
   \end{cases}
\end{eqnarray*}
The empirical loss, under the settings defined earlier, can now be defined as the negative of the log-likeihood function, given as:
\begin{eqnarray*}
        L_{BQR}(y;f_{\tau(x)}) = y_i \log((P(y=1|f_{\tau}(x_i))^{-1})) + \\ {(1-y_i)}\log( (1 - P(y=1|f_{\tau}(x_i)))^{-1})
\end{eqnarray*}
It is to be noted that, we can recover logistic and probit models, when the error distributions are logistic and normal distributions, respectively. Next, we analyze the learnability of latent functions. Before we do that, we provide two lemmas. (See Appendix \ref{appendix} for Proofs)
\begin{lem}
The Lipschitz constant of the BQR loss is $\max(\tau,1-\tau)$
\end{lem}
The implications of this important result are manifested in section 6. We set the learning rate accordingly and accomplish significantly faster convergence in binary classification tasks. It may also be useful in studying the robustness of BQR against adversarial attacks \cite{kevin}.
\begin{lem}
BQR also admits a bound in terms of the curvature of the function $f^*$. That is
$$
    c_1 E((f-f^*)^2) \le E(L(y,f)-L(y,f^*)) \le   c_2 E((f-f^*)^2)
$$
where $c_1$ and $c_2$ constants, bounded away from 0.
\end{lem}
Due to Lemmas 2.1 and 2.2, the BQR loss satisfies Eqn(2.1) of \cite{farrell}. Consequently, all the results of their paper are directly applicable, under suitable conditions. In particular, we restate their major result, Theorem 2:
\begin{thm}
Suppose Assumptions 2.1-2.3 hold. Let $f$ be the deep ReLU network with $W$ number of parameters. Under BQR, with probability at least $1-e^{-\gamma}$, for large enough $n$, for some $C>0$, 
$$ \|f-f^*\|^2_{L_2(x)} =  E((f-f^*)^2) \le B $$
for $ B = C\left( \frac{W\log(W)}{n}\log n+ \frac{\log\log n + r}{n} + \epsilon_{f^*}^2 \right)  $
\end{thm}
The above non-asymptotic error bounds can be used to tune and optimize the architectures as a function of the DNN architecture complexity $W$, sample size $n$, confidence $\gamma$ and approximation error $\epsilon_{f^*}$. Most refreshingly, it allows us to look at the DNN as a decompressor: given quantized outputs $y$, and inputs $x$, we can train a DNN to decompress the signal $f(x)$. 
\begin{figure}[htb!]
\centering
\includegraphics[width=0.7\columnwidth]{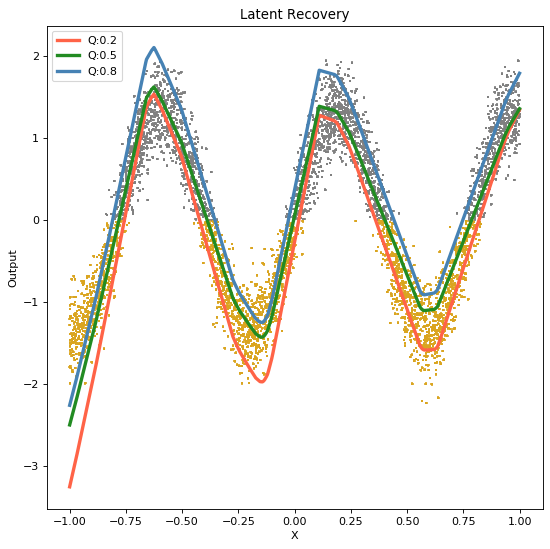}
\caption{Recovering the latent response $f_{\tau}(x)$ for D1}.
\label{fig:latent_d1}
\end{figure}
In Fig.~\ref{fig:latent_d1}, we show the estimated conditional quantiles of the latent response. It is fascinating to see that the original signal is recovered, despite observing only quantized labels at the time of training. In addition, it is worth noting that multiple quantiles are simultaneously estimated. In order to prevent quantile crossing, we add a regularization term,
$$L_{BQC} = \sum_{i=1}^{n} \sum_{p=1}^{m-1} \max(0,Q_{x_i}(\tau_{p}) -Q_{x_i}(\tau_{p+1}) $$ giving us regularized BQR:$L_{BQ} = L_{BQR} + \lambda L_{BQC}$. In this work, we find that for most cases a $\lambda=1$ was sufficient to prevent crossing. Next, we quantify how well the latent functions are learnt in terms of coverage, where coverage is an estimate of $P_{x,y}( x < Q_{x}(\tau) )$ that should be close to the nominal value $\tau$

\subsubsection{Coverage in Simulated Datasets}
To verify that the quantiles we obtain possess the coverage property, we scale the true latent distribution and the obtained quantiles. The threshold value is first subtracted from the true latent, and the resulting distribution is then normalized to have 0 mean and unit standard deviation. For the quantiles, the mean and standard deviation of the median is obtained, and all the quantiles are normalized using these terms. This normalization step is done solely for comparing against the latent.   
\par We created a collection of datasets as per the distributions listed below, sampling $X$ from $U(-1,1)$ and classified the points as class 0 if $y_i \leq \mu$. We computed the coverage for the generated quantiles. The datasets used are formulated as follows, with D5 and D6 being variants of the dataset proposed in \cite{Anand_2019_epsilon}. The results can be seen in Table \ref{tab:class_coverage_simulated_datasets}
\begin{itemize}
    \item D1 : $y_i = 5\sin{8x_i}  + \zeta_i$, where $\zeta_i \sim N(0,1)$
    \item D2 : $y_i = \nicefrac{(4x_i)^2}{2}  + \zeta_i$, where $\zeta_i \sim N(0,0.5)$
    \item D3 : $y_i = \sqrt{(4x_i)^2+5} - 2.5 + \zeta_i$, where $\zeta_i \sim U(-0.3,0.3)$ 
    
    \item D4 :$y_i= \zeta_i + 
    \begin{cases}
        2x_i \sin(1/2x_i) & x \neq 0 \\
        0 & x = 0 \\
    \end{cases}
    $, where $\zeta_i \sim N(0,0.5)$
    \item D5 :  $y_i = 2((1 - 3x_i+2(3x_i)^2)\exp{-0.5(3x_i)^2} -1.5)+ \zeta_i$, where $\zeta_i \sim N(0,0.25)$
    \item D6 :  $y_i = 2((1 - 3x_i+2(3x_i)^2)\exp{-0.5(3x_i)^2} -1.5) + \nicefrac{\zeta_i}{4}$, where $\zeta_i \sim \chi^2(2)$
\end{itemize}

\begin{table}[htb!]
    \centering
    \begin{tabular}{llllllllll}
    \toprule
    \multicolumn{1}{c}{} &  \multicolumn{9}{c}{Coverage for $\tau$ } \\
    \cmidrule(r){2-10}
    Dataset & 0.1 & 0.2 & 0.3 & 0.4 & 0.5 & 0.6 & 0.7 & 0.8 & 0.9 \\
    \midrule
    D1 & 0.10 & 0.18 & 0.28 & 0.40 & 0.51 & 0.59 & 0.69 & 0.81 & 0.92  \\
    D2 & 0.14 & 0.23 & 0.33 & 0.45 & 0.51 & 0.60 & 0.69 & 0.79 & 0.84 \\
    D3 & 0.10 & 0.19 & 0.28 & 0.39 & 0.51 & 0.58 & 0.69 & 0.81 & 0.86 \\
    D4 & 0.04 & 0.09 & 0.22 & 0.37 & 0.50 & 0.67 & 0.75 & 0.82 & 0.91 \\ 
    D5 & 0.08 & 0.20 & 0.36 & 0.46 & 0.53 & 0.61 & 0.70 & 0.82 & 0.89 \\ 
    D6 & 0.05 & 0.18 & 0.32 & 0.40 & 0.49 & 0.55 & 0.69 & 0.81 & 0.90 \\
    \bottomrule
    \end{tabular}
    \caption{Coverage values for simulated datasets}
    \label{tab:class_coverage_simulated_datasets}
\end{table}
\subsubsection{Coverage in Real World Datasets}
We use regression datasets taken from the UCI Machine Learning Repository \cite{Dua_UCI_Repo}, and convert them into classification tasks by thresholding the target, and converting it into a binary label. We use two different thresholds, one to simulate a balanced classification task, and the other to simulate an imbalanced problem. The results can be seen in Table \ref{tab:cov_uci_results}. The scaling methodology is the same as the method described for the simulated datasets. For both simulated and real-world datasets, we observe that reported coverages are very close to their nominal values around the median, and the precision decreases as the nominal quantile moves away from the median. While we do not know the distribution of the estimators, from the classical QR perspective, it suggests that, the precision at the lower quantiles is more dominated by the density terms, than by the $\tau(1-\tau)$ factor \cite{koenker_2005}.
\begin{table}[htb!]
  \centering
  \begin{tabular}{lllllllllllll}
    \toprule
    \multicolumn{4}{c}{} &  \multicolumn{9}{c}{Coverage for $\tau$ } \\
    \cmidrule(r){5-13}
    Dataset & $t$ & Acc. & RMSE & 0.1 & 0.2 & 0.3 & 0.4 & 0.5 & 0.6 & 0.7 & 0.8 & 0.9 \\
    \midrule
    Abalone & 9 & 0.81 & 0.83 & 0.08 & 0.19 & 0.27 & 0.35 & 0.55 & 0.68 & 0.79 & 0.88 & 0.97 \\
     & 7 & 0.89 & 0.89 & 0.08 & 0.19 & 0.31 & 0.43 & 0.54 & 0.65 & 0.75 & 0.84 & 0.96 \\
    \midrule
    Boston & 22 & 0.88 & 0.74 & 0.17 & 0.22 & 0.30 & 0.39 & 0.50 & 0.62 & 0.71 & 0.80 & 0.90 \\
    & 18 & 0.90 & 0.82 & 0.09 & 0.20 & 0.34 & 0.42 & 0.53 & 0.67 & 0.74 & 0.80 & 0.86 \\
    \midrule
    California & 180K & 0.80 & 0.77 & 0.05 & 0.15 & 0.26 & 0.36 & 0.51 & 0.62 & 0.75 & 0.84 & 0.93 \\
    & 200K & 0.79 & 0.83 & 0.10 & 0.25 & 0.37 & 0.48 & 0.57 & 0.66 & 0.73 & 0.79 & 0.87 \\
    \midrule
    Concrete & 35 & 0.88 & 0.62 & 0.09 & 0.17 & 0.26 & 0.36 & 0.50 & 0.64 & 0.76 & 0.88 & 0.94 \\
    & 50 & 0.91 & 0.66 & 0.11 & 0.18 & 0.28 & 0.41 & 0.50 & 0.65 & 0.79 & 0.83 & 0.87 \\
    \midrule
    Energy & 20 & 0.99 & 0.40 & 0.13 & 0.18 & 0.27 & 0.39 & 0.50 & 0.66 & 0.79 & 0.85 & 0.91 \\
    & 15 & 0.94 & 0.52 & 0.07 & 0.16 & 0.29 & 0.37 & 0.51 & 0.65 & 0.74 & 0.81 & 0.90 \\
    \midrule
    Protein & 5 & 0.82 & 0.82 & 0.10 & 0.22 & 0.34 & 0.44 & 0.53 & 0.63 & 0.73 & 0.83 & 0.93 \\
    & 9 & 0.81 & 0.84 & 0.09 & 0.16 & 0.30 & 0.42 & 0.53 & 0.62 & 0.72 & 0.82 & 0.92 \\
    \midrule
    Redshift & 0.65 & 0.91 & 0.83 & 0.09 & 0.18 & 0.26 & 0.37 & 0.48 & 0.61 & 0.81 & 0.86 & 0.92 \\
    & 0.9 & 0.92 & 0.88 & 0.07 & 0.10 & 0.15 & 0.32 & 0.45 & 0.70 & 0.77 & 0.88 & 0.96 \\
    \midrule
    Wine & 5 & 0.82 & 0.82 & 0.08 & 0.18 & 0.27 & 0.38 & 0.49 & 0.61 & 0.72 & 0.83 & 0.92 \\
    & 6 & 0.93 & 0.93 & 0.03 & 0.12 & 0.24 & 0.37 & 0.51 & 0.64 & 0.73 & 0.80 & 0.86 \\
    \midrule
    Yacht & 2 & 0.98 & 0.63 & 0.17 & 0.27 & 0.35 & 0.43 & 0.49 & 0.55 & 0.64 & 0.81 & 0.89 \\
    & 7.5 & 0.98 & 0.60 & 0.19 & 0.34 & 0.41 & 0.45 & 0.51 & 0.69 & 0.84 & 0.91 & 0.98 \\
    \bottomrule
  \end{tabular}
  \caption{Coverage results for binary classification using thresholded UCI regression datasets}
  \label{tab:cov_uci_results}
\end{table}

\section{Establishing Uncertainty in Classification with Quantiles}

\cite{nguyen2015deep} studied how Deep Learning models can be fooled easily, despite the high confidence in the predictions. 
\cite{Gal2016Dropout, kannan, Pragya} discuss why the widely used class probabilities, estimated with logistic loss, cannot be used as confidence measures, as they can overestimate the confidence, and are not consistent. One way to approach the problem is by Uncertainty Quantification (UQ). \cite{NIPS2017_Balaji} proposed Deep Ensembles with a Bayesian justification to report Monte Carlo estimates of prediction variance. \cite{Balaji_uq_2} posed UQ as a min-max problem, where a single model, instead of an ensemble, is input-distance aware.
\cite{natasa} proposed using Quantiles to report the PIs in the regression setting. It is straightforward to establish PIs using the conditional quantiles even in the binary classification setting since
$P_{x,y}( x < Q_{x}(\tau) ) = \tau$. It follows then that, 
$[Q_{x}(0.5\tau), Q_{x}(1-0.5\tau)]$ is a $100(1-\tau)$\% PI at $x$. Any monotonic transformation, such as a sigmoid function or an  Indicator function, can be used produce PIs in the class probabilities space or the label space. Along with measuring the precision, it is sometimes helpful to know when to withheld from a making prediction. Recently, \cite{maya} proposed \textit{Trustscore} based on how close a sample is to a set of high trustworthy samples to that affect. In addition to reporting the precision (via PIs), we can also measure the confidence via confidence score $\delta$ defined as follows:
\begin{defn}
 Confidence Score ($\delta$), defined for a sample $x$ as 
$$ \delta = \inf_{d \in (0,0.5)} \{d:  \text{ s.t } Q_{x}(0.5-d) \leq 0 \leq Q_{x}(0.5+d) \} $$
\end{defn}
The Confidence Score $\delta$ is thus a metric of how close to the decision boundary the latent function for $x_{i}$ is. As $\delta$ increases, the likelihood of the point being misclassified reduces, as the quantiles for the latent response move further away from the decision boundary.  The relationship between misclassiication rate and confidence can be explicitly stated as follows:
\begin{thm}
An instance with confidence score $\delta$ has a misclassification rate of $0.5- \delta$ 
\end{thm}
\begin{proof} Let $\mu$ be the median of the latent response $z$, i.e, $\mu = Q_x(0.5)$, and  that $\mu \ge 0$. Note that, 
\begin{equation*}
P(z<=\mu) = P(z<=0) + P(0 < z \le \mu)\\
\end{equation*}
By definition $P(z<=\mu)=0.5$, $\delta = P(0 < z <=\mu)$, and $P(z <=0)$ is the misclassification rate. Using same reasoning, we can show that, when $\mu < 0$, the misclasification rate $P(z>0)$ is $0.5-\delta$.  Hence, the misclassification rate is $0.5-\delta$.
\end{proof}
To verify the relationship between misclassification and confidence, we use the same datasets and thresholds used in our coverage computation tests, and compute the goodness of fit ($R^2$) score of the expected misclassification rate vs $\delta$ curve on the obtained values misclassification rate per delta, The results can be seen in Table \ref{tab:misc_vs_delta_artificial} \footnote{Note: The $R^2$ score for yacht is correct. The value is computed using scikit learn's \href{https://scikit-learn.org/stable/modules/model\_evaluation.html\#r2-score}{\texttt{r2\_score}}, which ranges from $-\infty$ to 1.0}. In addition, by omitting samples whose $\delta$-score is below a certain threshold, more confident predictions can be obtained, as per the theorem described above. We term this threshold as the  \textit{model confidence}. However, it is important to keep in mind that as this tolerance for a certain $\delta$ becomes more rigid, the number of acceptable decisions will also reduce. We define the \textit{retention rate} for a given confidence threshold as the ratio of number of points having a $\delta$-Score less than that confidence score to the samples available for decision. Table \ref{tab:delta_metrics_binaryclassification} shows the retention ($r_{r}$) and misclassification rates ($m_{r}$) for some standard binary classification datasets \cite{Dua_UCI_Repo,elson2007asirra,Kermany2018IdentifyingMD,IMDB}. 
\begin{table}[htb!]
    \centering
    \begin{tabular}{lllllllllll}
    \toprule
    \multicolumn{3}{c}{} &  \multicolumn{5}{c}{$\delta$-Score} \\
    \cmidrule(r){4-8}
    Dataset & t & Rate & 0.1 & 0.2 & 0.3 & 0.4 & 0.5 & $R^{2}$ \\
    \midrule
        Abalone & 9 & $m_{r}$ & 0.40 & 0.35 & 0.25 & 0.15 & 0.04 & 0.89\\
			&   & $r_{r}$& 1.00 & 0.84 & 0.68 & 0.52 & 0.33 & --	 \\
	        \cmidrule(r){2-9}
            & 7 & $m_{r}$ & 0.45 & 0.33 & 0.23 & 0.13 & 0.02 & 0.95\\
			&   & $r_{r}$ & 1.00 & 0.94 & 0.89 & 0.80 & 0.63 & 	--  \\
    \midrule
    Boston  & 22 & $m_{r}$ & 0.44 & 0.33 & 0.23 & 0.13 & 0.02 & 0.96\\
			&    & $r_{r}$ & 1.00 & 0.96 & 0.91 & 0.86 & 0.74 &   -- \\
			\cmidrule(r){2-9}
            & 18 & $m_{r}$ & 0.30 & 0.27 & 0.18 & 0.10 & 0.02 & 0.78\\
			&    & $r_{r}$ & 1.00 & 0.97 & 0.92 & 0.87 & 0.78 &  --  \\
    \midrule
    California & 1.8 & $m_{r}$ & 0.41 & 0.35 & 0.24 & 0.13 & 0.03 & 0.94 \\
			   &     & $r_{r}$ & 1.00 & 0.92 & 0.84 & 0.74 & 0.60 & --  \\
			   \cmidrule(r){2-9}
		       & 2.0 & $m_{r}$ & 0.43 & 0.37 & 0.26 & 0.15 & 0.03 & 0.88 \\
			   &	 & $r_{r}$ & 1.00 & 0.92 & 0.84 & 0.76 & 0.61 & -- \\ 
    \midrule
    Concrete & 35 & $m_{r}$ & 0.42 & 0.31 & 0.23 & 0.11 & 0.04 & 0.97\\
			 &	  & $r_{r}$ & 1.00 & 0.94 & 0.87 & 0.79 & 0.66 & -- \\
			 \cmidrule(r){2-9}
			 & 50 & $m_{r}$ & 0.46 & 0.34 & 0.20 & 0.15 & 0.01 & 0.94\\
			 & 	  & $r_{r}$ & 1.00 & 0.97 & 0.92 & 0.88 & 0.81 & --  \\
    \midrule
    Energy   & 20 & $m_{r}$ & 0.38 & 0.21 & 0.23 & 0.08 & 0.00 & 0.89\\
			 & 	  & $r_{r}$ & 1.00 & 0.99 & 0.99 & 0.99 & 0.97 & -- \\
			 \cmidrule(r){2-9}
			 & 15 & $m_{r}$ & 0.42 & 0.41 & 0.39 & 0.21 & 0.00 & 0.54\\
			 &	  & $r_{r}$ & 1.00 & 0.96 & 0.93 & 0.89 & 0.84 &  -- \\
    \midrule
    Protein  & 5 & $m_{r}$ & 0.37 & 0.36 & 0.24 & 0.13 & 0.04 & 0.90\\
			 & 	 & $r_{r}$ & 1.00 & 0.88 & 0.76 & 0.61 & 0.40 & --  \\ 
			 \cmidrule(r){2-9}
			 & 9 & $m_{r}$ & 0.41 & 0.37 & 0.26 & 0.15 & 0.04 & 0.87\\
			 &	 & $r_{r}$ & 1.00 & 0.88 & 0.75 & 0.61 & 0.40 & -- \\
    \midrule
    Redshift & 0.65 & $m_{r}$ & 0.40 & 0.32 & 0.20 & 0.12 & 0.02 & 0.99\\
			 &      & $r_{r}$ & 1.00 & 0.96 & 0.92 & 0.87 & 0.78 &  -- \\
			 \cmidrule(r){2-9}
			 & 0.9  & $m_{r}$ & 0.39 & 0.28 & 0.19 & 0.19 & 0.01 & 0.89\\
			 &      & $r_{r}$ & 1.00 & 0.97 & 0.93 & 0.87 & 0.81 & -- \\
    \midrule
    Wine 	 & 5 & $m_{r}$ & 0.43 & 0.39 & 0.29 & 0.17 & 0.07 & 0.70\\
			 &	 & $r_{r}$ & 1.00 & 0.87 & 0.74 & 0.57 & 0.33 &  --\\
			 \cmidrule(r){2-9}
			 & 6 & $m_{r}$ & 0.47 & 0.36 & 0.26 & 0.18 & 0.03 & 0.83\\
			 &	 & $r_{r}$ & 1.00 & 0.95 & 0.90 & 0.83 & 0.74	& --  \\
			 
    \midrule
    Yacht 	& 2   & $m_{r}$ & 0.13 & 0.10 & 0.03 & 0.00 & 0.00 & -9.6\\
			& 	  & $r_{r}$ & 1.00 & 1.00 & 0.99 & 0.97 & 0.89 & -- \\
			\cmidrule(r){2-9}
			& 7.5 & $m_{r}$ & 0.24 & 0.10 & 0.00 & 0.00 & 0.00 & -1.6\\ 
			& 	  & $r_{r}$ & 1.00 & 0.99 & 0.98 & 0.94 & 0.86 & -- \\

    \bottomrule
    \end{tabular}
    \caption{Misclassifaction rate per $\delta$-Score in Artificially created classification tasks}
    \label{tab:misc_vs_delta_artificial}
\end{table}

\begin{table}[htbp!]
    \centering
    \begin{tabular}{llllllll}
    \toprule
    \multicolumn{3}{c}{} &  \multicolumn{5}{c}{Confidence Scores($\delta$)} \\
    \cmidrule(r){4-8}
    Dataset & Acc. & Metric & 0.1 & 0.2 & 0.3 & 0.4 & 0.5\\
    \midrule
    Asirra & 0.95 & $m_{r}$& 0.39 & 0.33 & 0.24 & 0.12 & 0.01\\
		   &	  & $r_{r}$ & 1.00 & 0.98 & 0.96 & 0.92 & 0.86\\
	\midrule
    Banknote & 1.00 & $m_{r}$ &0.10 & 0.00 & 0.01 & 0.00 & 0.00\\
			 &	    & $r_{r}$ &1.00 & 0.96 & 0.92 & 0.86 & 0.78\\
	\midrule
    Haberman & 0.76 & $m_{r}$ &0.43 & 0.39 & 0.36 & 0.19 & 0.12\\
			 &      & $r_{r}$ &1.00 & 0.88 & 0.77 & 0.66 & 0.26\\
	\midrule
    Heart Disease & 0.88 & $m_{r}$ &0.46 & 0.36 & 0.23 & 0.12 & 0.04\\
			      &      & $r_{r}$ &1.00 & 0.92 & 0.85 & 0.74 & 0.58\\
	\midrule
    Indian Liver Patient & 0.75 & $m_{r}$ &0.41 & 0.42 & 0.32 & 0.22 & 0.03\\
						 &      & $r_{r}$ &1.00 & 0.87 & 0.58 & 0.42 & 0.32\\
	\midrule
	IMDB     & 0.92 &$m_r$& 0.45 & 0.29 & 0.20 & 0.11 & 0.01\\
	         &      &$r_r$& 1.00 & 0.98 & 0.97 & 0.95 & 0.91\\
	\midrule
    Ionosphere & 0.94 & $m_{r}$ &0.50 & 0.32 & 0.17 & 0.16 & 0.02\\
			   &      & $r_{r}$ &1.00 & 0.97 & 0.95 & 0.90 & 0.82\\
	\midrule
    Pima Indian Diabetes & 0.80 & $m_{r}$ &0.45 & 0.37 & 0.30 & 0.16 & 0.04\\
						 &      & $r_{r}$ &1.00 & 0.88 & 0.74 & 0.58 & 0.37\\
	\midrule
    Pneumonia X-Ray & 0.89 & $m_{r}$ &0.72 & 0.31 & 0.24 & 0.20 & 0.04\\
					&      & $r_{r}$ &1.00 & 0.99 & 0.97 & 0.95 & 0.92\\
	\midrule
    Sonar & 0.94 & $m_{r}$ &0.45 & 0.38 & 0.22 & 0.13 & 0.03\\
		  &      & $r_{r}$ &1.00 & 0.96 & 0.92 & 0.87 & 0.80\\
	\midrule
    Titanic & 0.87 & $m_{r}$ &0.41 & 0.30 & 0.27 & 0.15 & 0.06\\
			&      & $r_{r}$ &1.00 & 0.94 & 0.88 & 0.78 & 0.65\\
	\midrule
    Wisconsin Breast Cancer & 0.97 & $m_{r}$ &0.36 & 0.28 & 0.20 & 0.11 & 0.01\\
							&      & $r_{r}$ &1.00 & 0.99 & 0.97 & 0.95 & 0.91 \\
    \bottomrule
    \end{tabular}
    \caption{Confidence score based metrics for binary classification datasets}
    \label{tab:delta_metrics_binaryclassification}
\end{table}

In addition, we can use the tried and tested classifier metrics, namely AUC, RoC and Precision-Recall curves in order to evaluate the classifier per Confidence Score level. We simply compute the TPR-FPR and Precision-Recall curves for the classifier, only considering points that have a specific confidence level. Figure \ref{fig:conf_auc} shows an example of the same. As one can note, the classifier performance improves when low confidence labels are withheld. The per $\delta$-score performance can be evaluated and used as another metric when deciding whether or not a prediction should be rejected.

\begin{figure}[htb!]
\centering
\includegraphics[width=0.7\columnwidth]{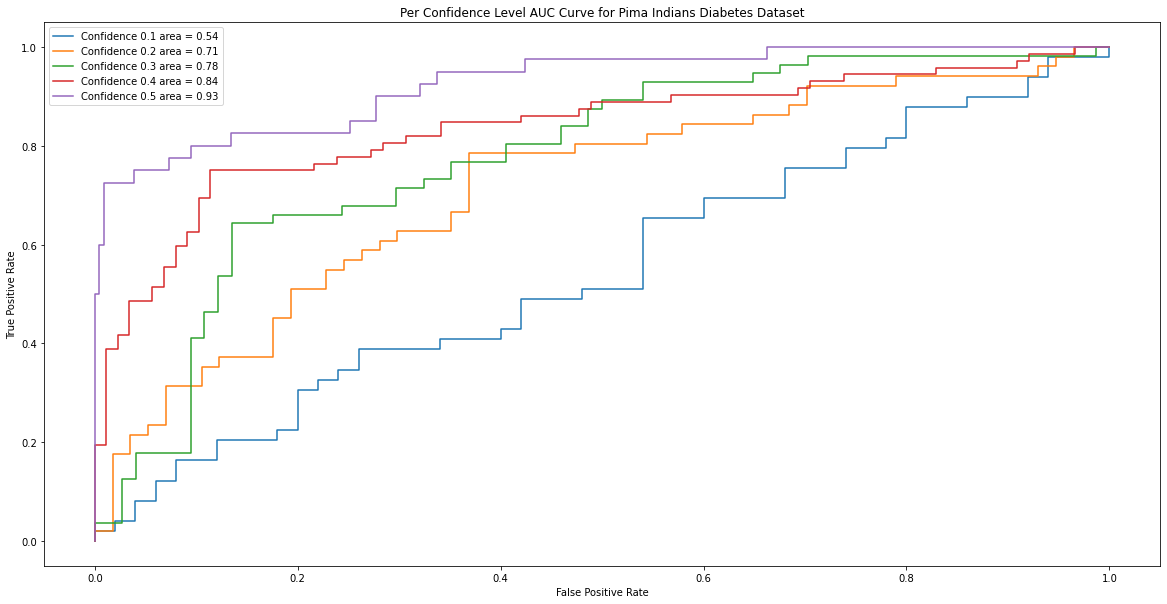}
\caption{Per $\delta$-Score AUC-ROC Curves}
\label{fig:conf_auc}
\end{figure}

\subsection{Comparison with Trustscore}
To compare our $\delta_x$ with TS, we compute both of them for all the samples in a dataset. Following this, we rank the samples based on TS and create 10 equi-distributed bins, one bin per decile.  We the compute the average $\delta_x$ and TS of all points in each bin. As per \cite{maya}, low ranking points are the ones likely to be misclassified, As seen in Table\ref{tab:delta_v_trust_withscores}, as TS bin decile increases, the average $\delta_x$ score also increases as expected, indicating that our method captures the expected trend. 

\begin{table}[htb!]
    \centering
    \begin{tabular}{llllllllllll}
    \toprule
    \multicolumn{2}{c}{} &  \multicolumn{10}{c}{Trustscore Bin} \\
    \cmidrule(r){3-12}
    Dataset & Avg. & 1 & 2 & 3 & 4 & 5 & 6 & 7 & 8 & 9 & 10 \\
    \midrule
    Banknote & $\delta_x$ & 0.49 & 0.50 & 0.50 & 0.50 & 0.50 & 0.50 & 0.50 & 0.50 & 0.50 & 0.50\\
            & TS & 3.69   &  5.61   &  6.95   &  8.22   &  9.50 & 11.12 & 13.01 & 15.69 & 20.08 & 1.79E11\\
    \midrule
    Haberman & $\delta_x$ & 0.22 & 0.24 & 0.25 & 0.28 & 0.32 & 0.36 & 0.37 & 0.41 & 0.44 & 0.45\\
            & TS & 0.60 & 0.94 & 1.13 & 1.39 & 1.71 & 2.20 & 2.83 & 3.73 & 5.45 & 5.3E11\\
    \midrule
    Heart & $\delta_x$ & 0.35 & 0.40 & 0.45 & 0.45 & 0.46 & 0.47 & 0.47 & 0.49  & 0.49 & 0.49  \\
         & TS &  0.86 & 1.08 & 1.19 & 1.27 & 1.36 & 1.46 & 1.59 & 1.81  & 2.19 & 2.3E11\\
    \midrule
    ILP & $\delta_x$ & 0.16 & 0.19 & 0.25 & 0.28 & 0.29 & 0.30 & 0.34 & 0.37 & 0.41 & 0.43\\   
        & TS & 0.51 & 0.81 & 1.00 & 1.14 & 1.26 & 1.42 & 1.64 & 1.89 & 2.38 & 5.1E11\\
    \midrule
    Iono & $\delta_x$ & 0.46 & 0.49 & 0.49 & 0.48 & 0.49 & 0.50 & 0.50 & 0.50 & 0.50 & 0.50\\
         & TS & 0.91 & 1.05 & 1.13 & 1.37 & 1.77 & 2.34 & 3.11 & 4.64 & 6.46 & 2.5E11\\
    \midrule
    Pima & $\delta_x$ & 0.23 & 0.26 & 0.30 & 0.37 & 0.38 & 0.40 & 0.41 & 0.45 & 0.47 & 0.48\\
        & TS &  0.79 & 0.99 & 1.11 & 1.22 & 1.33 & 1.45 & 1.59 & 1.78 & 2.09 & 3.22\\ 
    \midrule
    Sonar & $\delta_x$ & 0.44 & 0.48 & 0.49 & 0.49 & 0.48 & 0.49 & 0.49 & 0.49 & 0.49 & 0.49 \\
          & TS & 0.94 & 1.06 & 1.12 & 1.20 & 1.30 & 1.39 & 1.50 & 1.63 & 2.03 & 2.54 \\
    \midrule
    Titanic & $\delta_x$ & 0.27 & 0.38 & 0.41 & 0.44 & 0.45 & 0.46 & 0.45 & 0.46 & 0.46 & 0.47\\
            & TS & 0.78 & 1.39 & 2.32 & 4.25 & 6.34 & 10.97 & 25.55 & 1.1E9 & 6.3E10 & 8.6E11\\
    \midrule
    WBC & $\delta_x$ & 0.43 & 0.48 & 0.49 & 0.50  & 0.50 & 0.50 & 0.50 & 0.50 & 0.50 & 0.50\\
       & TS &1.19 & 1.61 & 1.95 & 2.49 & 3.55 & 8.3E10 & 1.6E12 & 1.8E12 & 2.1E12 & 2.4E12\\ 
    \bottomrule
    \end{tabular}
    \caption{Average $\delta_x$-score and Trustscore values per Trustscore bin}
    \label{tab:delta_v_trust_withscores}
\end{table}

However, note that, the average TS per bin can increase dramatically, as TS is not calibrated and is not intended to be used in isolation, but rather as a relative metric.  Further, the real benefit of our approach can be seen in datasets such as Ionosphere, Banknote and Sonar, wherein even the lowest TS bins have very high $\delta_x$ values. This makes perfect sense, as these are easy to classify datasets. However, by construction, TS involves ranking the samples, and as a result, it is possible for some samples with low confidence to get high TS and vice versa, which is undesirable. Unlike TS, $\delta_x$ is well calibrated due to Theorem 3.2, and requires no other models to be run on the top, thanks to the conditional quantiles.

\section{Robustness to Label Noise}
As noted previously in Section 3, BCE can be overconfident or can be fooled easily in the presence of noise. While confidence scores may help detect such spurious cases, is it possible to reduce them in the first place? That brings us to developing robust estimation techniques.  It is widely established that Quantiles are a type of robust estimators to noise in the response variable (vertical noise). In the classification setting, \cite{sastry} showcase the ability of the MAE of the class probabilities being more robust to label noise than Categorical Cross Entropy but note that training under MAE could be slow. In our case, having shown the promise of BQR in UQ, we are interested in seeing whether median class probabilities are robust to label noise. To study this, we use the same networks as before, with one trained on BCE, and the other BQR. For each dataset, we vary the percentage of wrongly labelled samples in the training set, and compare the accuracy of the model on the entire real dataset. The results can be seen in Table \ref{tab:label_noise}. At no noise, the BQR based classifier is usually equivalent to, or slightly less accurate than the BCE one, however, as noise increases (20\% and above), eventually the BQR classifier begins to outperform the BCE one, but it is to be noted that these noise levels are extremely high when this occurs
\begin{small}
\begin{table}[htb!]
  \centering
  \begin{tabular}{lllllll}
    \toprule
    \multicolumn{2}{c}{} &  \multicolumn{5}{c}{$\%$ of Flipped Labels } \\
    \cmidrule(r){3-7}
    Dataset & Loss & 0\% & 10\% & 20\% & 30\% & 40\%\\
    \midrule
    Banknote & BCE & 1.000 & 1.000 & 0.998 & 0.986 & 0.925 \\
	 & BQR & 1.000 & 0.999 & 0.997 & 0.989 & 0.939\\
    \midrule
    Haberman & BCE & 0.767 & 0.766 & 0.744 & 0.733 & 0.642 \\
	 & BQR & 0.764 & 0.763 & 0.749 & 0.735 & 0.688 \\
    \midrule
    Heart & BCE & 0.919 & 0.881 & 0.822 & 0.723 & 0.645 \\
	 & BQR & 0.899 & 0.859 & 0.836 & 0.785 & 0.700 \\
    \midrule
    Ionosphere & BCE & 0.962 & 0.923 & 0.881 & 0.799 & 0.666 \\
	 & BQR & 0.950 & 0.916 & 0.887 & 0.841 & 0.706 \\
    \midrule
    Pima & BCE & 0.817 & 0.803 & 0.776 & 0.714 & 0.618 \\
	 & BQR & 0.802 & 0.792 & 0.776 & 0.735 & 0.683 \\
    \midrule
    Sonar & BCE & 0.957 & 0.880 & 0.786 & 0.700 & 0.586 \\
	 & BQR & 0.946 & 0.875 & 0.801 & 0.718 & 0.607 \\
    \midrule
    Titanic & BCE & 0.874 & 0.868 & 0.858 & 0.828 & 0.756 \\
	 & BQR & 0.872 & 0.866 & 0.859 & 0.845 & 0.805 \\
    \midrule
    WBC & BCE & 0.978 & 0.970 & 0.964 & 0.930 & 0.841 \\
	 & BQR & 0.975 & 0.970 & 0.967 & 0.951 & 0.917 \\
    \bottomrule
  \end{tabular}
  \caption{BCE vs BQR Loss on Label Noise}
  \label{tab:label_noise}
\end{table}
\end{small}

\section{Explaining the predictions}
A strong criticism of Deep Learning models is that they are black-boxes in nature \cite{Rudin}. A wide variety of techniques that attempt to explain the model predictions in terms of activations, saliency maps, and counter-factuals based on gradient propagation are proposed and are actively being developed \cite{shrikumar, saliency, sundar}. In majority of the cases, the same techniques can be applied to DNNs fit with BQR as well. There are additional classes of explanations for mean predictions like  \texttt{shapley} \cite{shapely} and \texttt{LIME} \cite{LIME}. Below, we show, how the conditional quantiles can be used to estimate conditional effects, as well as report conditional means. Recall that $Q_{x}(\tau)$ is the conditional quantile at the covariate $x$, where $\tau$ is chosen over a set of  discrete values $ T \in \{\tau_0, \tau_1, \tau_2,\hdots,\tau_m,\tau_{m+1}\}$ with $\tau_0=0, \tau_{m+1}=1$, and there are $m$ outputs available from the neural net corresponding to the remaining values of $\tau$. We can get a smoothed version of the conditional quantiles by:
\begin{eqnarray*}
Q^s_{x}(\tau)_{\tau \in (0,1) } = \sum_{i=0}^m Q_x(\tau_{i}) \int_{p=\tau_{i}}^{\tau_{i+1}} \frac{1}{h} K(\frac{\tau-p}{h}) dp
\end{eqnarray*}
where $K(.)$ is suitable kernel with bandwidth parameter $h$ \cite{parzen_79}. In our examples, we used a Gaussian kernel with bandwidth set to 0.1. Now, one can immediately compute any univariate statistic. In particular, the mean response can be computed as: $E(f(x)) = \int_{\tau=0}^1 Q^s_{x}(\tau) d\tau$. Likewise, $Var(f(x))$ can also be computed. In fact, any quantity of interest can be computed simply by post-processing the smoothed full distribution \cite{parzen_unified}.

\begin{figure}[htb!]
\centering
\includegraphics[width=0.7\columnwidth]{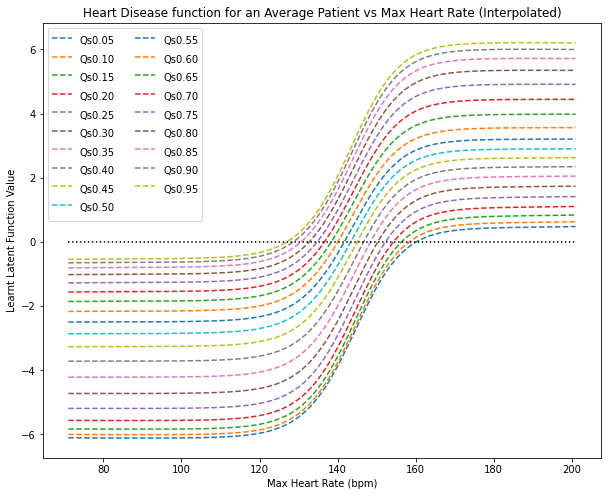}
\caption{Heart Disease Latent vs Max Heart Rate for an Average patient (Interpolated)}
\label{fig:heart_latent_interpolated}
\end{figure}

\begin{figure}[htb!]
\centering
\includegraphics[width=0.7\columnwidth]{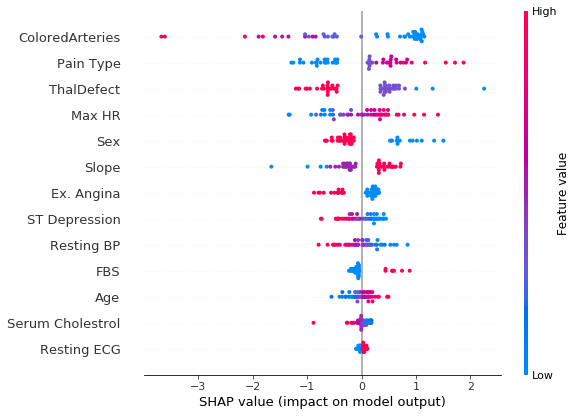}
\caption{An Example of the obtained \texttt{Shapely} statistics of the mean response of the Heart Disease Dataset}
\label{fig:shap_results}
\end{figure}

Figure \ref{fig:heart_latent_interpolated} showcases how quantiles can be used to aid in explainable predictions. In these graphs, the average metrics of a patient in the heart disease dataset were computed, and the quantiles were predicted using these average parameters while varying the maximum heart rate from the the minimum recorded value to the maximum, in steps of 1. The figure graphically showcases the region of uncertainty, something which cannot be obtained from conventional binary classifiers as they provide only a single threshold value. Figure \ref{fig:shap_results} showcases the \texttt{shapely} summary statistics of the mean response of the latent on the test data of the heart disease dataset via quantile interpolation.
\par  The smoothed quantiles also show how it is possible obtain more fine grained values of the confidence metric $\delta$, while keeping the number of prediction quantile outputs manageable.

\section{Lipschitz Adaptive Learning Rates: BQR Loss}
A critical parameter in training DNNs via Stochastic Gradient Descent is the learning rate. One of the early approaches to adapt the learning rate is by recognizing the inverse relationship between the step size of gradient descent update and the Lipschitz constant of the function being optimized \cite{armijo1966}. Since the Lipschitz constant is generally unknown apriori, \cite{Michael_97,Michael_99} estimate a local approximation during training the feed forward networks. Recently works have derived adaptive Learning Rates by exploiting the gradient properties of Deep ReLU networks \cite{Saha_ALR, Saha_ALR_Check} , and successfully applied them to train models on large datasets \cite{Saha_Parsimonious}. 
We summarize their results in the following proposition : 
\begin{prop}
In a Deep ReLU network, let constant $k_z$ be the supremum of gradients w.r.t the function, and let $L$ be the Lipschitz constant of the Loss. Then, the adaptive Learning rate $\eta$ is
$\eta = {(k_z L)}^{-1}$, where the weight update rule is:
$w^{t} = w^{t-1} -  \eta \nabla L(f(x))$
\end{prop}
This particular choice of LALR, under the assumption that gradients cannot change arbitrarily fast, ensures a convex quadratic upper bound, minimized by the descent step.

To show the efficacy of the Lipschitz constant based adaptive learning rate, we compared the performance of the adaptive learning rate verses fixed learning rates of 0.01 and 0.1, and tested how quickly they were able to reach a specified target accuracy in terms of number of epochs. The results can be found in Table \ref{tab:adaptive_vs_fixed}. N/A indicates that the classifier was unable to reach the accuracy threshold within 5000 Epochs (500 for IMDB) - if this occurs, the maximum accuracy reached is provided as well. For IMDB, we use an embedding dimension of size 100, and a 2 layer LSTM of dimension 256 which feeds a linear layer.

\begin{table}[htbp!]
    \centering
    \begin{tabular}{llllll}
    \toprule
    Dataset & Accuracy & $N_{0.01}$ & $N_{0.1}$ & $N_{1/L}$\\
    \midrule
    Banknote & 0.99 & 945 & 104 & 14 \\
    Haberman & 0.80 & N/A (0.775) & 773 & 78\\
    Heart & 0.85 & 221 & 161 & 4\\
    ILP & 0.75 & N/A (0.733) & 317 & 28\\
    IMDB & 0.90 & N/A (0.716) & 106 & 27\\
    Ionosphere & 0.90 & 1841 & 104 & 6\\
    Pima & 0.80 & 4099 & 417 & 53\\
    Sonar & 0.97 & 2199 & 320 & 40\\
    Titanic & 0.87 & 982 & 152 & 17\\
    Winsconsin & 0.97 & 1577 & 101 & 8\\
    \bottomrule
    \end{tabular}
    \caption{Convergence comparison between the Adaptive and Fixed Learning rates for SGD}
    \label{tab:adaptive_vs_fixed}
\end{table}

For our image datasets, we compared the efficacy of the LALR on various deep architectures. The Resnet \cite{He2016DeepRL} implementations are Pytorch's default Resnet18 and Resnet50 implementations of 18 and 50 layers each, while the Densenet \cite{Huang2017DenselyCC} implementation is Pytorch's Densenet121 architecture consisting of 10 layers, out of which 4 are Dense blocks, for a total of 121 layers. For all models, the optimizer was SGD, and each test was run for 20 epochs. We obtained the best validation accuracy of LR=0.01 and found the number of epochs for the other learning rates to achieve both a training and validation accuracy equal to or greater than LR=0.01. The results can be seen in Table \ref{tab:deep_image_lr}. For all our tests, the LALR based models converged faster, barring Asirra on Densenet.

\begin{table}[htbp!]
    \centering
    \begin{tabular}{llllll}
    \toprule
    Dataset & Arch. & Target Acc. & LR & $N_{E}$ & $T_{E}$ (min)\\
    \midrule
    Asirra & Resnet18 & 0.76  & Fixed (0.01) & 16 & 2.5\\
           &          & 	  & Fixed (0.1)  & 15 & 2.5\\
           &          & 	  & Adaptive     & 6  & 4.2\\
           \cmidrule(r){2-6}
           & Resnet50 & 0.70 & Fixed (0.01)  & 20 & 4.7\\
           &          & 	 & Fixed (0.1)   & 16 & 4.7\\
           &          & 	 & Adaptive      & 5  & 7.0\\
           \cmidrule(r){2-6}
           & Densenet & 0.86 & Fixed (0.01) & 20 & 4.9\\
           &          & 	 & Fixed (0.1)  & 8  & 4.9\\
           &          & 	 & Adaptive     & 9  & 12.1\\
    \midrule
     Pneumonia & Resnet18 & 0.83 & Fixed (0.01) & 18 & 1.2\\
			   &          & 	 & Fixed (0.1)  & 7  & 1.2\\
               &          & 	 & Adaptive     & 6  & 2.3\\
               \cmidrule(r){2-6}
               & Resnet50 & 0.82 & Fixed (0.01) & 20 & 1.9\\
               &          & 	 & Fixed (0.1) 	& 15 & 1.9\\
               &          &      & Adaptive 	& 9  & 3.2\\
               \cmidrule(r){2-6}
               & Densenet & 0.81 & Fixed (0.01) & 20 & 1.9\\
               &          & 	 & Fixed (0.1)  & 10 & 1.9\\
               &          & 	 & Adaptive     & 3  & 3.2\\
    \bottomrule
    \end{tabular}
    \caption{Adaptive LR performance in Deep Binary Image Classification}
    \label{tab:deep_image_lr}
\end{table}

\section{Conclusion and Scope for Future Work}
To summarize, in this work we put forth the Binary Quantile Regression loss function - a loss function to allow DNNs for binary classification to learn the quantiles of the latent function learnt by the network. By estimating these quantiles, we are able to gain additional insight into the uncertainty of the predictions of the network in real time. To further this, we also extend this uncertainty quantification technique to the sample confidence score we term $\delta_x$. We show that $\delta_x$ is an accurate measure to capture uncertainty, that provides a mathematical likelihood of misclassification, as per the function learnt by the model. Following this, we explore how quantiles provide solutions to current open problems in the deep learning space by being more robust to extreme amounts label noise and allowing for standard function explanation techniques to be applied to DNN outputs. Finally, we show how recent advances in adaptive learning rates can be applied to BQR as well.

To conclude, BQR allows us to enhance binary classification networks by providing additional information at prediction times, with no impact on performance. The quantile outputs obtained have a variety of use cases, the most potent of which is the ability to provide the uncertainty metric we describe.

One glaring limitation however, is of course the fact that BQR applies only to binary classification tasks, and cannot be used directly in a multivariate setting. One way of overcoming this drawback is to use a One-vs-all approach, akin to multiclass classification with SVMs. Alternatively, when the classes are ordinal, we can extend the binary thresholded model to include more cut-points. Otherwise, it is not trivial to extend BQR to multi-class setting because there is no unique way to define multivariate quantiles. However, recent research in depth quantiles \cite{probal_median,hallin2010} could allow BQR to be extended in this direction.

\section*{Acknowledgement}
The authors would like to thank the Science and Engineering Research Board (SERB), Department of Science and Technology, Government of India, for supporting our research by providing us with resources to conduct our experiments. The project reference number is: EMR/2016/005687. The authors are indebted to Prof. Probal Choudhury, ISI Kolkata for suggestions and critical insights which helped the manuscript immensely. Anuj and Anirudh would like to thank Inumella Sricharan for his help, and Dr. K S Srinivas of the Department of CS\&E, PES University for his help, advice and encouragement.

\bibliographystyle{plain}
\bibliography{refs}

\begin{thebibliography}{10}

\bibitem{Ahmed2020AFF}
Shakkeel Ahmed, Ravi~S. Mula, and Soma~S. Dhavala.
\newblock A framework for democratizing ai.
\newblock {\em ArXiv}, abs/2001.00818, 2020.

\bibitem{Anand_2019_epsilon}
Pritam Anand, Reshma Rastogi, and Suresh Chandra.
\newblock A new asymmetric $\epsilon$-insensitive pinball loss function based
  support vector quantile regression model, 2019.

\bibitem{Michael_97}
George~S. Androulakis, Michael~N. Vrahatis, and George~D. Magoulas.
\newblock Effective backpropagation training with variable stepsize.
\newblock {\em Neural Networks}, 10(1):69—82, January 1997.

\bibitem{armijo1966}
Larry Armijo.
\newblock Minimization of functions having lipschitz continuous first partial
  derivatives.
\newblock {\em Pacific J. Math.}, 16(1):1--3, 1966.

\bibitem{Benoit_12}
Dries~F. Benoit and Dirk Van~den Poel.
\newblock Binary quantile regression: a bayesian approach based on the
  asymmetric laplace distribution.
\newblock {\em Journal of Applied Econometrics}, 27(7):1174--1188, 2012.

\bibitem{trivedi_2010}
A.~Colin Cameron and Pravin~K. Trivedi.
\newblock {\em Microeconometrics Using Stata, Revised Edition}.
\newblock Stata Press, 2nd edition, 2010.

\bibitem{probal_l1}
Probal Chaudhuri.
\newblock Generalized regression quantiles : forming a useful toolkit for
  robust linear regression.
\newblock In {\em In L1 Statistical Analysis and Related Methods –
  Proceedings of the Second International Conference on L1 Norm and Related
  Methods}, page 169–185. North Holland : Amsterdam, 1992.

\bibitem{probal_median}
Probal Chaudhuri.
\newblock On a geometric notion of quantiles for multivariate data.
\newblock {\em Journal of the American Statistical Association}, 91:862--872,
  1996.

\bibitem{probal_doksum}
Probal Chaudhuri, K~Doksum, and A.~Samarov.
\newblock On average derivative quantile regression.
\newblock {\em Ann. Statist}, 25(2):715--744, 1997.

\bibitem{Dua_UCI_Repo}
Dheeru Dua and Casey Graff.
\newblock {UCI} machine learning repository, 2017.

\bibitem{elson2007asirra}
Jeremy Elson, John~(JD) Douceur, Jon Howell, and Jared Saul.
\newblock Asirra: A captcha that exploits interest-aligned manual image
  categorization.
\newblock In {\em Proceedings of 14th ACM Conference on Computer and
  Communications Security (CCS)}, October 2007.

\bibitem{farrell}
Max~H. Farrell, Tengyuan Liang, and Sanjog Misra.
\newblock Deep neural networks for estimation and inference: Application to
  causal effects and other semiparametric estimands.
\newblock {\em ArXiv}, 2018.

\bibitem{Gal2016Dropout}
Yarin Gal and Zoubin Ghahramani.
\newblock Dropout as a {B}ayesian approximation: Representing model uncertainty
  in deep learning.
\newblock In {\em Proceedings of the 33rd International Conference on Machine
  Learning (ICML-16)}, 2016.

\bibitem{sastry}
Aritra Ghosh, Himanshu Kumar, and P.~S. Sastry.
\newblock Robust loss functions under label noise for deep neural networks.
\newblock In {\em Proceedings of the Thirty-First AAAI Conference on Artificial
  Intelligence}, AAAI'17, page 1919–1925, 2017.

\bibitem{hallin2010}
Marc Hallin, Davy Paindaveine, and Miroslav Šiman.
\newblock Multivariate quantiles and multiple-output regression quantiles: From
  l 1 optimization to halfspace depth.
\newblock {\em Ann. Statist.}, 38(2):635--669, 04 2010.

\bibitem{He2016DeepRL}
Kaiming He, X.~Zhang, Shaoqing Ren, and Jian Sun.
\newblock Deep residual learning for image recognition.
\newblock {\em 2016 IEEE Conference on Computer Vision and Pattern Recognition
  (CVPR)}, pages 770--778, 2016.

\bibitem{horowitz_92}
J.~L Horowitz.
\newblock A smoothed maximum score estimator for the binary response model.
\newblock {\em Econometrica}, 60:505--531, 1992.

\bibitem{Huang2017DenselyCC}
Gao Huang, Zhuang Liu, and Kilian~Q. Weinberger.
\newblock Densely connected convolutional networks.
\newblock {\em 2017 IEEE Conference on Computer Vision and Pattern Recognition
  (CVPR)}, pages 2261--2269, 2017.

\bibitem{smola}
Takeuchi Ichiro, V.~Le Quoc, D.~Sears Timothy, and J.~Smola Alexander.
\newblock Nonparametric quantile estimation.
\newblock {\em Journal of Machine Learning Research}, 7:1231--64, 07 2006.

\bibitem{maya}
H~Jiang, B~Kim, M~Guan, and M~Gupta.
\newblock To trust or not to trust a classifier.
\newblock {\em Advances in Neural Information Processing Systems 31}, 2018.

\bibitem{Kermany2018IdentifyingMD}
Daniel~S. Kermany, Michael~H. Goldbaum, Wenjia Cai, Carolina Carvalho~Soares
  Valentim, and Kang Zhang.
\newblock Identifying medical diagnoses and treatable diseases by image-based
  deep learning.
\newblock {\em Cell}, 172:1122--1131.e9, 2018.

\bibitem{koenker_78}
R.~Koenker and G.~B. Bassett.
\newblock Regression quantiles.
\newblock {\em Econometrica}, 46:33--50, 1978.

\bibitem{koenker_2005}
Roger Koenker.
\newblock {\em Quantile Regression}.
\newblock Econometric Society Monographs. Cambridge University Press, 2005.

\bibitem{kordas_2006}
G~Kordas.
\newblock Smoothed binary regression quantiles.
\newblock {\em Journal of Applied Economics}, 21:387--407, 2006.

\bibitem{NIPS2017_Balaji}
Balaji Lakshminarayanan, Alexander Pritzel, and Charles Blundell.
\newblock Simple and scalable predictive uncertainty estimation using deep
  ensembles.
\newblock In {\em Proceedings of the 31st International Conference on Neural
  Information Processing Systems}, NIPS'17, page 6405–6416, Red Hook, NY,
  USA, 2017. Curran Associates Inc.

\bibitem{hinton}
Yann LeCun, Y.~Bengio, and Geoffrey Hinton.
\newblock Deep learning.
\newblock {\em Nature}, 521:436--44, 05 2015.

\bibitem{Balaji_uq_2}
Jeremiah~Zhe Liu, Zi~Lin, Shreyas Padhy, Dustin Tran, Tania Bedrax-Weiss, and
  Balaji Lakshminarayanan.
\newblock Simple and principled uncertainty estimation with deterministic deep
  learning via distance awareness, 2020.

\bibitem{shapely}
Scott~M. Lundberg, Gabriel Erion, Hugh Chen, Alex DeGrave, Jordan~M. Prutkin,
  Bala Nair, Ronit Katz, Jonathan Himmelfarb, Nisha Bansal, and Su-In Lee.
\newblock From local explanations to global understanding with explainable ai
  for trees.
\newblock {\em Nature Machine Intelligence}, 2(1):56--67, Jan 2020.

\bibitem{IMDB}
Andrew~L. Maas, Raymond~E. Daly, Peter~T. Pham, Dan Huang, Andrew~Y. Ng, and
  Christopher Potts.
\newblock Learning word vectors for sentiment analysis.
\newblock In {\em Proceedings of the 49th Annual Meeting of the Association for
  Computational Linguistics: Human Language Technologies}, pages 142--150,
  Portland, Oregon, USA, June 2011. Association for Computational Linguistics.

\bibitem{mansky_75}
C.~F Mansky.
\newblock Maximum score estimation of the stochastic utility model of choice.
\newblock {\em Journal of Economics}, 3:205--228, 1975.

\bibitem{mansky_85}
C.~F Mansky.
\newblock Semiparametric analysis of discrete response: Asymptotic porperties
  of the maximum score estimator.
\newblock {\em Journal of Economics}, 3:205--228, 1985.

\bibitem{Maronna2006RobustST}
Ricardo Maronna, Douglas Martin, and Victor Yohai.
\newblock {\em Robust Statistics: Theory and Methods}.
\newblock Wiley, 06 2006.

\bibitem{kannan}
Kamil Nar, Orhan Ocal, S.~Shankar Sastry, and Kannan Ramchandran.
\newblock Cross-entropy loss and low-rank features have responsibility for
  adversarial examples.
\newblock {\em CoRR}, abs/1901.08360, 2019.

\bibitem{nguyen2015deep}
Anh Nguyen, Jason Yosinski, and Jeff Clune.
\newblock Deep neural networks are easily fooled: High confidence predictions
  for unrecognizable images.
\newblock {\em 2015 IEEE Conference on Computer Vision and Pattern Recognition
  (CVPR)}, 2015.

\bibitem{Michael_99}
V.P. Palgianakos, Michael Vrahatis, and George Magoulas.
\newblock Nonmonotone methods for backpropagation training with adaptive
  learning rate.
\newblock {\em International Joint Conference on Neural Networks}, 3:1762 --
  1767, 02 1999.

\bibitem{parzen_79}
Emanuel Parzen.
\newblock Nonparametric statistical data modeling.
\newblock {\em Journal of the American Statistical Association},
  74(365):105--121, 1979.

\bibitem{parzen_unified}
Emanuel Parzen.
\newblock Quantile probability and statistical data modeling.
\newblock {\em Statistical Science}, 19, 11 2004.

\bibitem{Portnoy_89}
Stephen Portnoy and Roger. Koenker.
\newblock Adaptive $l$-estimation for linear models.
\newblock {\em Ann. Statist.}, 17(1):362--381, 1989.

\bibitem{LIME}
Marco~Tulio Ribeiro, Sameer Singh, and Carlos Guestrin.
\newblock "why should i trust you?": Explaining the predictions of any
  classifier.
\newblock In {\em Proceedings of the 22nd ACM SIGKDD International Conference
  on Knowledge Discovery and Data Mining}, KDD '16, page 1135–1144, New York,
  NY, USA, 2016. Association for Computing Machinery.

\bibitem{Rudin}
Cynthia Rudin.
\newblock Stop explaining black box machine learning models for high stakes
  decisions and use interpretable models instead.
\newblock {\em Nature Machine Intelligence}, 1(5):206--215, May 2019.

\bibitem{Saha_ALR_Check}
S~Saha, Tejas Prashanth, Suraj Aralihalli, Sumedh Basarkod, T.~S.~B. Sudarshan,
  and Soma.~S. Dhavala.
\newblock Lalr: Theoretical and experimental validation of lipschitz adaptive
  learning rate in regression and neural networks.
\newblock {\em International Joint Conference on neural Networks},
  abs/2006.13307, 2020.

\bibitem{kevin}
Kevin Scaman and Aladin Virmaux.
\newblock Lipschitz regularity of deep neural networks: Analysis and efficient
  estimation.
\newblock In {\em Advances in Neural Information Processing Systems 32}, page
  3839–3848, 2018.

\bibitem{shrikumar}
Avanti Shrikumar, Peyton Greenside, and Anshul Kundaje.
\newblock Learning important features through propagating activation
  differences.
\newblock In {\em Proceedings of the 34th International Conference on Machine
  Learning - Volume 70}, ICML'17, page 3145–3153. JMLR.org, 2017.

\bibitem{saliency}
Karen Simonyan, Andrea Vedaldi, and Andrew Zisserman.
\newblock Deep inside convolutional networks: Visualising image classification
  models and saliency maps, 2014.

\bibitem{Saha_Parsimonious}
Shailesh Sridhar, Snehanshu Saha, A.~Shaikh, Rahul Yedida, and Sriparna Saha.
\newblock Parsimonious computing: A minority training regime for effective
  prediction in large microarray expression data sets.
\newblock {\em International Joint Conference on neural Networks},
  arxive.org/abs/2005.08442, 2020.

\bibitem{sundar}
Mukund Sundararajan, Ankur Taly, and Qiqi Yan.
\newblock Gradients of counterfactuals.
\newblock {\em ArXiv}, abs/1611.02639, 2016.

\bibitem{Pragya}
Pragya Sur and Emmanuel~J. Cand{\`e}s.
\newblock A modern maximum-likelihood theory for high-dimensional logistic
  regression.
\newblock {\em Proceedings of the National Academy of Sciences},
  116(29):14516--14525, 2019.

\bibitem{natasa}
Natasa Tagasovska and David Lopez-Paz.
\newblock Single-model uncertainties for deep learning.
\newblock In {\em Advances in Neural Information Processing Systems}, pages
  6417--6428, 2019.

\bibitem{Saha_ALR}
Rahul Yedida and Snehanshu Saha.
\newblock Lipschitzlr: Using theoretically computed adaptive learning rates for
  fast convergence.
\newblock {\em Applied Intelligence}, arxiv.org/abs/1902.07399, 2020.

\bibitem{zhu_cqr}
Hui Zou and Ming Yuan.
\newblock Composite quantile regression and the oracle model selection theory.
\newblock {\em Annals of Statistics}, 36(3):1108--1126, 2008.

\end{thebibliography}

\newpage
\begin{appendices}
\section{Proofs} \label{appendix}
\textbf{Lemma 2.1}
\begin{lem*}
The Lipschitz constant of the BQR loss is $\max(\tau,1-\tau)$
\end{lem*}

\begin{proof}
Recall that, the empirical risk under the BQR loss is:
\begin{eqnarray*}
    L(y,z) = -(y\log{p_z}+ (1-y)\log{(1-p_z)})
\end{eqnarray*} where
\begin{eqnarray*}
    p_z &\equiv&
    \begin{cases} 
      1 - \tau \exp((\tau-1)z) & z \ge 0  \\
      (1-\tau)\exp(\tau z)) & z < 0 
   \end{cases}
\end{eqnarray*}
Let us consider the following cases.
\emph{Case-1a}: $0 < z_1 < z_2, y=1$
\begin{equation*}
\frac{|L(1,z_2)-L(1,z_1)|}{|z_2-z_1|} = \frac{\log{(1 - \tau e^{(\tau-1)z_2})} -\log{(1-\tau e^{(\tau-1)z_1})}} {z_2-z_1}
\end{equation*}
The RHS approaches maximum as $z_2,z_1 \xrightarrow[]{}0 $. Taking the limit w.r.t $z_1$ first, we get,
\begin{equation*}
\lim_{z1\to 0} \frac{|L(1,z_2)-L(1,z_1)|}{|z_2-z_1|} = \frac{ \log{(1-\tau e^{(\tau-1)z_1})}-\log{(1 - \tau)}} {z_2}
\end{equation*} and then taking the limit w.r.t $z_2$ later, we get
\begin{eqnarray*}
    \lim_{z_2,z_1\to 0} \frac{|L(1,z_2)-L(1,z_1)|}{|z_2-z_1|} &=& \tau
\end{eqnarray*}
Therefore, 
\begin{eqnarray*}
    \frac{|L(1,z_2)-L(1,z_1)|}{|z_2-z_1|} &\le& \tau
\end{eqnarray*}
\emph{Case-1b}: $0 < z_1 < z_2, y=0$
\begin{equation*}
\frac{|L(0,z_2)-L(0,z_1)|}{|z_2-z_1|} = \frac{\log{(\tau e^{(\tau-1)z_2})} -\log{(\tau e^{(\tau-1)z_1})}} {z_2-z_1}
\end{equation*}
In this case, the RHS simplifies to, 
\begin{eqnarray*}
    \frac{|L(0,z_2)-L(0,z_1)|}{|z_2-z_1|} &=& \frac{(z_2-z_1)(1-\tau)}{z_2-z_1}\\
    &=& 1-\tau
\end{eqnarray*}
Therefore, 
\begin{eqnarray*}
    \frac{|L(0,z_2)-L(0,z_1)|}{|z_2-z_1|} &\le& 1-\tau
\end{eqnarray*}

\emph{Case-2a}: $ z_1 < 0 < z_2, y=1$
\begin{equation*}
\frac{|L(1,z_2)-L(1,z_1)|}{|z_2-z_1|} = \frac{\log{(1 - \tau e^{(\tau-1)z_2})} -\log{((1-\tau) e^{\tau z_1})}} {z_2-z_1}
\end{equation*}
The RHS approaches maximum as $z_2,z_1 \xrightarrow[]{}0 $. Taking the limit w.r.t $z_1$ first, we get,
\begin{equation*}
\lim_{z1\to 0} \frac{|L(1,z_2)-L(1,z_1)|}{|z_2-z_1|} = \frac{ \log{(1-\tau e^{(\tau-1)z_2})}-\log{(1 - \tau)}} {z_2}
\end{equation*} and then taking the limit w.r.t $z_2$, we get
\begin{eqnarray*}
    \lim_{z_2,z_1\to 0} \frac{|L(1,z_2)-L(1,z_1)|}{|z_2-z_1|} &=& \tau
\end{eqnarray*}
Therefore, 
\begin{eqnarray*}
    \frac{|L(1,z_2)-L(1,z_1)|}{|z_2-z_1|} &\le& \tau
\end{eqnarray*}
\emph{Case-2b}: $z_1 < 0 < z_2, y=0$
\begin{equation*}
\frac{|L(0,z_2)-L(0,z_1)|}{|z_2-z_1|} = \frac{\log{(\tau e^{(\tau-1)z_2})} -\log{(1-(1-\tau) e^{\tau z_1})}} {z_2-z_1}
\end{equation*}
The RHS approaches maximum as $z_2,z_1 \xrightarrow[]{}0 $. Taking the limit w.r.t $z_1$ first, we get,
\begin{equation*}
\lim_{z1\to 0} \frac{|L(0,z_2)-L(0,z_1)|}{|z_2-z_1|} = \frac{ \log{(1-\tau e^{(\tau-1)z_2})}-\log{(\tau)}} {z_2}
\end{equation*} and then taking the limit w.r.t $z_2$, we get
\begin{eqnarray*}
    \lim_{z_2,z_1\to 0} \frac{|L(0,z_2)-L(0,z_1)|}{|z_2-z_1|} &=& 1-\tau
\end{eqnarray*}
Therefore, 
\begin{eqnarray*}
    \frac{|L(0,z_2)-L(0,z_1)|}{|z_2-z_1|} &\le& 1-\tau
\end{eqnarray*}
\emph{Case-3a}: $ z_1  < z_2 < 0, y=1$
\begin{equation*}
\frac{|L(1,z_2)-L(1,z_1)|}{|z_2-z_1|} = \frac{\log{((1 - \tau) e^{\tau z_2})} -\log{((1-\tau) e^{\tau z_1})}} {z_2-z_1}
\end{equation*}
The RHS simplifies to,
\begin{eqnarray*}
    \frac{|L(1,z_2)-L(1,z_1)|}{|z_2-z_1|} &=& \frac{\tau (z_2-z_1)} {z_2-z_1}
\end{eqnarray*}
Therefore, 
\begin{eqnarray*}
    \frac{|L(0,z_2)-L(0,z_1)|}{|z_2-z_1|} &\le& \tau
\end{eqnarray*}
\emph{Case-3b}: $z_1 < z_2 <0,  y=0$
\begin{equation*}
\frac{|L(0,z_2)-L(0,z_1)|}{|z_2-z_1|} = \frac{\log{(1-(1 - \tau) e^{\tau z_2})} -\log{(1-(1-\tau) e^{\tau z_1})}} {z_1-z_2}
\end{equation*}
The RHS approaches maximum as $z_1,z_2 \xrightarrow[]{}0 $. Taking the limit w.r.t $z_2$ first, we get,
\begin{equation*}
\lim_{z_1\to 0} \frac{|L(0,z_2)-L(0,z_1)|}{|z_2-z_1|} = \frac{\log{(1 - \tau)} -\log{(1-\tau e^{(\tau-1)z_1})}} {z_1}
\end{equation*} and then taking the limit w.r.t $z_1$, we get
\begin{eqnarray*}
    \lim_{z_1,z_2\to 0} \frac{|L(0,z_2)-L(0,z_1)|}{|z_2-z_1|} &=& 1-\tau
\end{eqnarray*}
Therefore, 
\begin{eqnarray*}
    \frac{|L(0,z_2)-L(0,z_1)|}{|z_2-z_1|} &\le& 1-\tau
\end{eqnarray*}
Hence, $\forall z_1, z_2 \in R, y \in \{0,1\}$
\begin{eqnarray*}
 \frac{|L(y,z_2)-L(y,z_1)|}{|z_2-z_1|} \le \max({1-\tau, \tau)}
\end{eqnarray*}

\end{proof}
\newpage
\textbf{Lemma 2.2}
\begin{lem*}
BQR also admits a bound in terms of the curvature of the function $f^*$. That is
$$
    c_1 E((f-f^*)^2) \le E(L(y,f)-L(y,f^*)) \le   c_2 E((f-f^*)^2)
$$
where $c_1$ and $c_2$ constants, bounded away from 0.
\end{lem*}
\begin{proof}
Recall that, the empirical risk under the BQR loss is:
\begin{eqnarray*}
    L(y,z) = -(y\log{p_z}+ (1-y)\log{(1-p_z)})
\end{eqnarray*} where
\begin{eqnarray*}
    p_z &\equiv&
    \begin{cases} 
      1 - \tau \exp((\tau-1)z) & z > 0  \\
      (1-\tau)\exp(\tau z)) & z \le 0
   \end{cases}
\end{eqnarray*}

Using Taylor series expansion of $h_a(b)=L(b,y)-L(a,y)$, with $a=f,b=f^*$, we can write,
\begin{eqnarray*}
    h_a(b) = h_a(a)+h_a'(a)(b-a)+\frac{1}{2}h_a''(a)(b-a)^2
\end{eqnarray*}
We will be looking at $h_a''$ to determine the bounds for the curvature of the loss function.
Let us consider the following cases.\\
\emph{Case-1}: $b\ge0, a\ge 0$
\begin{eqnarray*}
    h_a(b) &=& -(1-te^{-(1-t)a})\log(1-te^{-(1-t)b}) - (te^{-(1-t)a})\log(te^{-(1-t)b})  +g(a)\\
    &=& -(1-te^{-(1-t)a})\log(1-te^{-(1-t)b}) - (te^{-(1-t)a}) (\log(t)-(1-t)b))  +g(a)
\end{eqnarray*}
\begin{eqnarray*}
    h'_a(b) &=& -(1-te^{-(1-t)a})\frac{t(1-t)e^{-(1-t)b}}{(1-te^{-(1-t)b})} - (te^{-(1-t)a})(1-t)
\end{eqnarray*}
\begin{eqnarray*}
    h''_a(b) &=& (1-te^{-(1-t)a})\frac{t(1-t)^2e^{-(1-t)b}}{(1-te^{-(1-t)b})^2}
\end{eqnarray*}
$h''_a(b)$ is maximum at $a=0,b=0$, and minimum at $a=M,b=M$, therefore
\begin{eqnarray*}
    A_1 \equiv h''_a(b) &\ge& \frac{t(1-t)^2e^{-(1-t)M}}{1-te^{-(1-t)M}}\\
    h''_a(b) &\le& t(1-t)
\end{eqnarray*}

\emph{Case-2}: $b\le0, a\le 0$
\begin{eqnarray*}
    h_a(b) &=& -(1-t)e^{ta}\log((1-t)e^{tb}) - (1-(1-t)e^{ta})\log(1-(1-t)e^{tb}) +g(a)\\
     &=& -(1-t)e^{ta}(\log((1-t)+tb) - (1-(1-t)e^{ta})\log(1-(1-t)e^{tb}) +g(a)
\end{eqnarray*}
\begin{eqnarray*}
    h'_a(b) &=& -(1-t)te^{ta} + t(1-t)(1-(1-t)e^{ta}) \frac{e^{tb}}{1-(1-t)e^{tb}}
\end{eqnarray*}
\begin{eqnarray*}
    h''_a(b) &=& t^2(1-t)(1-(1-t)e^{ta})\frac{e^{tb}}{(1-(1-t)e^{tb})^2}
\end{eqnarray*}
$h''_a(b)$ is maximum at $a=0,b=0$, and minimum at $a=-M,b=-M$, therefore
\begin{eqnarray*}
    A_2 \equiv h''_a(b) &\ge& \frac{t^2(1-t)e^{-tM}}{1-(1-t)e^{-tM}}\\
    h''_a(b) &\le& t(1-t)
\end{eqnarray*}

\emph{Case-3}: $b\ge0, a\le 0$
\begin{eqnarray*}
    h_a(b) &=& -(1-t)e^{ta}\log(1-te^{-(1-t)b}) - (1-(1-t)e^{ta})\log(te^{-(1-t)b}) +g(a)\\
    &=& -(1-t)e^{ta} - \log(1-te^{-(1-t)b}) - (1-(1-t)e^{ta})(\log(t)+-(1-t)b) +g(a)\\
\end{eqnarray*}
\begin{eqnarray*}
    h'_a(b) &=&-t(1-t)^2e^{ta}\frac{e^{-(1-t)b}} {1-te^{-(1-t)b}} + (1-(1-t)e^{ta}(1-t)
\end{eqnarray*}
\begin{eqnarray*}
    h''_a(b) &=& t(`-t)^3e^{ta}\frac{e^{-(1-t)b}}{(1-te^{-(1-t)b})^2}
\end{eqnarray*}
$h''_a(b)$ is maximum at $a=0,b=0$, and minimum at $a=-M,b=M$, therefore
\begin{eqnarray*}
    A_3 \equiv h''_a(b) &\ge& \frac{t(1-t)^3e^{-M}}{(1-te^{-(1-t)M})^2}\\
    h''_a(b) &\le& t(1-t)
\end{eqnarray*}

\emph{Case-4}: $b\le0, a\ge 0$
\begin{eqnarray*}
    h_a(b) &=& -(1-te^{-(1-t)a})\log((1-t)e^{tb}) - te^{-(1-t)a}\log(1-(1-t)e^{tb}) +g(a)\\
    &=&  -(1-te^{-(1-t)a})\log((1-t)e^{tb}) - (1-te^{-(1-t)a})(\log(1-t)+tb)
\end{eqnarray*}
\begin{eqnarray*}
    h'_a(b) &=&  -(1-te^{-(1-t)a})t - t^2(1-t)e^{-(1-t)a}\frac{e^{tb}}{1-(1-t)e^tb}
\end{eqnarray*}
\begin{eqnarray*}
    h''_a(b) &=& t^3(1-t)e^{-(1-t)a} \frac{e^{tb}}{(1-(1-t)e^{tb})^2}
\end{eqnarray*}
$h''_a(b)$ is maximum at $a=0,b=0$, and minimum at $a=M,b=-M$, therefore
\begin{eqnarray*}
    A_4 \equiv h''_a(b) &\ge& \frac{t^3(1-t)e^{-M}}{(1-(1-t)e^{-M})^2}\\
    h''_a(b) &\le& t(1-t)
\end{eqnarray*}
Therefore,
\begin{eqnarray*}
    c_1 &=& 0.5\min{(A_1,A_2,A_3,A_4)}\\
    c_2 &=& 0.5t(1-t)
\end{eqnarray*}
\end{proof}

\textbf{Theorem 2.3}
\begin{thm*}
Suppose Assumptions 2.1-2.3 hold. Let $f$ be the deep ReLU network with $W$ number of parameters. Under BQR, with probability at least $1-e^{-\gamma}$, for large enough $n$, for some $C>0$, 
$$ \|f-f^*\|^2_{L_2(x)} =  E((f-f^*)^2) \le B $$
for $ B = C\left( \frac{W\log(W)}{n}\log n+ \frac{\log\log n + r}{n} + \epsilon_{f^*}^2 \right)  $
\end{thm*}
\begin{proof}
See Theorem 2 of \cite{farrell}
\end{proof}

\end{appendices}
\end{document}